\def\year{2022}\relax
\documentclass[letterpaper]{article} 
\usepackage{aaai22}  
\usepackage{times}  
\usepackage{helvet}  
\usepackage{courier}  
\usepackage[hyphens]{url}  
\usepackage{graphicx} 
\usepackage{natbib}  
\usepackage{caption} 
\DeclareCaptionStyle{ruled}{labelfont=normalfont,labelsep=colon,strut=off} 
\usepackage{algorithm}
\usepackage{algorithmic}

\usepackage{newfloat}
\usepackage{listings}
\floatstyle{ruled}
\newfloat{listing}{tb}{lst}{}
\floatname{listing}{Listing}
\usepackage{subcaption}
\usepackage{footnote}
\makesavenoteenv{tabular}
\usepackage{doi}
\usepackage{amsmath}
\usepackage{amssymb}
\usepackage{mathtools}
\usepackage{siunitx}
\usepackage{cases}
\usepackage{enumitem}
\usepackage{booktabs}
\usepackage{multirow}

\newcommand{\R}{\mathbb{R}}

\newcommand{\Sym}{\operatorname{Sym}}
\newcommand{\norm}[1]{\left\lVert#1\right\rVert}

\newcommand{\set}[1]{\{#1\}}
\newcommand{\inp}[1]{\langle #1 \rangle}
\newtheorem{theorem}{Theorem}[section]

\newtheorem{lemma}[theorem]{Lemma}

\newtheorem{remark}[theorem]{Remark}
\newenvironment{proof}[1][Proof]{\textbf{#1.} }{\ \rule{0.5em}{0.5em}}

\title{Feedback Gradient Descent:\\ Efficient and Stable Optimization with Orthogonality for DNNs}

\author{
    Fanchen Bu, Dong Eui Chang
}
\affiliations{

    School of Electrical Engineering, KAIST \\
    \{boqvezen97, dechang\}@kaist.ac.kr
}

\begin{document}

\maketitle

\begin{abstract}
The optimization with orthogonality has been shown useful in training deep neural networks (DNNs). 
To impose orthogonality on DNNs, both computational efficiency and stability are important.
However, existing methods utilizing Riemannian optimization or hard constraints can only ensure stability while those using soft constraints can only improve efficiency.
In this paper, we propose a novel method, named Feedback Gradient Descent (FGD), to our knowledge, the \textit{first} work showing \textit{high efficiency and stability simultaneously}.
FGD induces orthogonality based on the simple yet indispensable Euler discretization of a continuous-time dynamical system on the tangent bundle of the Stiefel manifold.
In particular, inspired by a numerical integration method on manifolds called Feedback Integrators, we propose to instantiate it on the tangent bundle of the Stiefel manifold for the \textit{first} time.
In the extensive image classification experiments, FGD comprehensively outperforms the existing state-of-the-art methods in terms of accuracy, efficiency, and stability.
\end{abstract}

\section*{Introduction}
During the prosperous and ongoing development of deep neural networks (DNNs), it has been shown that imposing orthogonality on the parameters can help improve the network performance, which has attracted substantial attention with theoretical analyses \cite{saxe2013exact, desjardins2015natural, harandi2016generalized}.

Many researchers have been trying to practically impose orthogonality during the training of DNNs.
As orthogonality can be regarded as a property of the Stiefel manifold, Riemannian optimization is a typical and direct technique to maintain orthogonality, which has been widely studied in the optimization field \cite{smith1994optimization, edelman1998geometry, rapcsak2002minimization, absil2009optimization, absil2012projection, wen2013feasible, bonnabel2013stochastic, jiang2015framework}.
However, most Riemannian optimization algorithms are computationally expensive due to the complex retraction or projection. They are especially computationally prohibitive when applied to the state-of-the-art DNNs that have numerous parameters \cite{ozay2018training}, which has been demonstrated in the comparisons done by a previous work \cite{li2020efficient}.
Without great loss of orthogonality, an alternative is to find a transformation or mapping, keeping the hard constraints on orthogonality and allowing normal Euclidean optimization on the transformed or mapped parameters \cite{huang2018orthogonal, li2020efficient}. Although these algorithms are much faster than Riemannian optimization algorithms, they are still considerably time-consuming when applied to DNNs.
To further improve the efficiency, another intuitive way is to apply soft constraints \cite{rodriguez2016regularizing, xie2017all, bansal2018can, jia2017improving, huang2020controllable, Wang_2020_CVPR} by adding a penalty term into the loss function. Nevertheless, such a way fails to maintain orthogonality during the training process. Therefore, it is imperative to develop a new algorithm with both high efficiency and high numerical stability.

In this paper, we propose a novel method, named Feedback Gradient Descent (FGD), to our knowledge, the \textit{first} work showing \textit{high efficiency and stability simultaneously}.
FGD induces orthogonality based on the simple Euler discretization of a continuous-time dynamical system (CTDS).
In particular, we start from building a CTDS on the tangent bundle of the Stiefel manifold to represent the gradient descent (GD) process \textit{with momentum} and orthogonality by following the idea of the Euclidean counterparts, i.e., the corresponding CTDSs of GD \cite{baldi1995gradient}, especially GD \textit{with momentum} \cite{qian1999momentum}.

To construct the corresponding optimization algorithm, it is indispensable yet nontrivial to discretize the CTDS on the tangent bundle.
The direct application of usual discretization techniques such as Euler's method for ordinary differential equations is inadequate because the trajectories of the discrete-time system can go off the manifold.
Meanwhile, usual structure-preserving discretization methods such as variational or symplectic integrators \cite{haier2006geometric} are complex and computationally expensive.
Inspired by Feedback Integrators (FI) \cite{chang2016feedback}, which is a numerical integration method for CTDSs on manifolds, we convert the discretization problem on the tangent bundle to one on a Euclidean space.
Specifically, following the framework of FI, we extend the CTDS representing GD on the tangent bundle to an ambient Euclidean space and modify the extended CTDS by adding a feedback term pulling the variables back to the tangent bundle so that the tangent bundle becomes a local attractor.
Moreover, we prove a new theorem, Theorem \ref*{thm:main_thm}, to directly show exponential stability while FI only guarantees asymptotic stability.
Although several applications of FI have been proposed \cite{chang2018controller, chang2019feedback, chang2019model, ko2021tracking, park2021transversely}, we are the first to instantiate the FI framework for a CTDS on the tangent bundle of the Stiefel manifold, as well as the first to apply the framework to an optimization problem, especially the GD problem with orthogonality.
Since the constructed CTDS is in a Euclidean space, we can efficiently discretize it without using any time-consuming operations with the stability carried over to the discretized system.

In summary, the major theoretical contributions are fourfold.
\begin{itemize}[leftmargin=*]
    \item To the best of our knowledge, FGD is the \textit{first} optimization algorithm that imposes \textit{orthogonality} with \textit{momentum} for DNNs based on a \textit{CTDS};
    \item We are the \textit{first} to instantiate the FI framework for a CTDS on the tangent bundle of the \textit{Stiefel manifold}, and also the first to apply the framework to the \textit{GD} problem;
    \item We generalize and improve the idea of FI, providing a new theorem (Theorem \ref*{thm:main_thm}) showing \textit{exponential stability} rather than the original asymptotic stability;
    \item FGD provides a way to apply optimization \textit{directly} on the parameters and to maintain their \textit{orthogonality} without using any time-consuming operations.
\end{itemize}

We conduct extensive experiments on the image classification task. We use a range of widely-used DNNs such as WideResNet \cite{zagoruyko2016wide}, ResNet \cite{he2016deep, he2016identity}, VGG \cite{simonyan2014very}, and PreActResNet \cite{he2016identity}, and the popular benchmark datasets such as CIFAR-10/100 \cite{krizhevsky2009learning}, SVHN \cite{netzer2011reading}, and ImageNet \cite{deng2009imagenet}. We compare FGD with stochastic gradient descent (SGD) \cite{robbins1951stochastic, kiefer1952stochastic} \textit{with momentum} and several state-of-the-art methods imposing soft or hard orthogonality \cite{rodriguez2016regularizing, jia2017improving, huang2018orthogonal, bansal2018can, li2020efficient, Wang_2020_CVPR, huang2020controllable}. FGD consistently outperforms the other baseline methods in terms of accuracy with favorable efficiency and stability. Especially, FGD achieves the highest accuracy in most cases and consumes less or similar training time than those of the methods using soft constraints. It also shows strong numerical stability comparable to previous methods imposing hard constraints.

\section*{Related Work}\label{sec_rel_work}
During the development of neural networks, orthogonality was first shown to be useful in mitigating the vanishing or exploding gradients problem \cite{bengio1994learning}, especially on recurrent neural networks (RNNs) \cite{pascanu2013difficulty, le2015simple, wisdom2016full, arjovsky2016unitary, jing2017tunable, hyland2017learning, vorontsov2017orthogonality, helfrich2020eigenvalue}.
To improve the efficiency of the optimization algorithms with orthogonality, many techniques have been utilized, e.g., householder reflections \cite{mhammedi2017efficient}, Cayley transform \cite{helfrich2018orthogonal, maduranga2019complex}, and exponential-map-based parameterization \cite{lezcano2019trivializations, lezcano2019cheap}.

Orthogonality has also shown with theoretical foundation its ability to help the training of general deep neural networks (DNNs) \cite{saxe2013exact, harandi2016generalized, liu2021orthogonal}, especially convolutional neural networks (CNNs) \cite{Wang_2020_CVPR, trockman2021orthogonalizing, liu2021convolutional}, by reducing the feature redundancy \cite{chen2017training}, stabilizing the distribution of activations over layers \cite{desjardins2015natural}, and so on.
Moreover, orthogonality has also been substantiated to be helpful in the training of generative adversarial networks (GANs) \cite{brock2016neural, miyato2018spectral, brock2018large, huang2020controllable}.
It is also notable that even merely the initialization with orthogonality is helpful \cite{mishkin2015all}.
One may expect to directly apply the existing methods on RNNs to general DNNs and CNNs. However, the methods on RNNs are usually limited to square matrices of small size, which makes it difficult to apply them directly to general DNNs and CNNs with numerous sizable parameters, especially when the computational resources are limited. Many researchers have been trying to propose practically applicable optimization algorithms on DNNs, especially CNNs, with orthogonality, which we focus on in this paper.

Naturally, as orthogonality can be seen as a property of the Stiefel manifold, a straightforward way is to directly utilize Riemannian optimization algorithms on the Stiefel manifold. However, although this field has continuous theoretical development, the direct application of Riemannian optimization requires computationally expensive retraction or projection to keep the parameters on the manifold \cite{ozay2018training}, which significantly increases the training time as demonstrated in a previous work \cite{li2020efficient} and can impair the stability of training \cite{harandi2016generalized, huang2018orthogonal, huang2020controllable}.

\noindent \textbf{Hard Constraints.} To address this problem without great loss of orthogonality, some previous methods, as some researchers have done on RNNs, use a parameter transformation or mapping from the Stiefel manifold to the Euclidean space, which still keeps hard constraints on orthogonality.
For example, a previous work \cite{huang2018orthogonal} finds a closed-form solution on the Stiefel manifold that minimizes the distance to the parameter on Euclidean space to be updated, where computationally expensive and numerically unstable eigendecomposition is required, which is also analyzed and mentioned in their later work \cite{huang2020controllable};
another work \cite{li2020efficient} proposes an iterative estimation of the retraction mapping based on the Cayley transform, which still needs considerable additional time as shown in the theoretical analyses and the practical experiments.

\noindent \textbf{Soft Constraints.} Alternatively and intuitively, one can use soft constraints by adding a penalty term in the loss function during the training to impose orthogonality with a low computational cost, which, nevertheless, means that the algorithms can hardly guarantee the maintenance of orthogonality, just like other regularization-based methods.
Many different kinds of penalty terms have been proposed, such as regularization based on the Frobenius norm and its variants \cite{rodriguez2016regularizing, xie2017all, bansal2018can}, regularization based the singular value \cite{jia2017improving, huang2020controllable}, and regularization based on doubly block Toeplitz matrices \cite{Wang_2020_CVPR}.

Unlike all the above existing methods, our method does not use any retraction, projection, transformation, or mapping but performs the optimization directly on the orthogonal parameters.
    
\section*{Preliminaries}

We use $\inp{\cdot, \cdot}$ to denote the Euclidean inner product of matrices, i.e., 
$\inp{A, B} = \operatorname{tr}(A^T B)$,
for any two matrices $A$ and $B$ of the same size, and we use $\norm{\cdot}$ to denote the norm induced from the above inner product.
We use $\Sym$ to denote the symmetrization operator defined by 
$\Sym(A) = \frac{1}{2}(A + A^T)$,
for any square matrix $A$.

The Stiefel manifold with parameters $n$ and $p$ consists of orthonormal $p$-frames in $\R^n$.
Formally, in this paper, we use the compact Stiefel manifold that is defined as
$\mathrm{St}(n,p) = \set{X \in \R^{n \times p}: X^T X = I_p},$
where $n \geq p$ and $I_p$ is the $p\times p$ identity matrix. Taking $\mathrm{St}(n,p)$ as an embedded submanifold of $\R^{n \times p}$ and using the Euclidean inner product, the tangent space of $\mathrm{St}(n,p)$ at $X\in \mathrm{St}(n,p)$ is expressed as 
$T_X \mathrm{St}(n,p) = \set{Z \in \R^{n \times p}: \Sym(X^T Z) = 0}$,
whose union over all points of the Stiefel manifold is the tangent bundle of the Stiefel manifold given by
$T \mathrm{St}(n,p) = \set{(X,Y): X \in \mathrm{St}(n,p), Y \in T_X \mathrm{St}(n,p)}$.
For any $M \in \R^{n \times p}$, we can project $M$ onto the tangent space $T_X \mathrm{St}(n,p)$ at $X$ by 
$f_{X}(M) = M - X\Sym(X^T M),$
thus the Riemannian gradient of a differentiable function $F$ is
$\operatorname{grad} F(X) 
= \nabla F(X) - X\Sym(X^T \nabla F(X)),$
where $\nabla F(\cdot)$ denotes the Euclidean gradient of $F$ as a function on $\R^{n \times p}$.

\section*{Feedback Gradient Descent}\label{sec:FGD}
Consider a generic optimization problem. Let $\mathcal{D}$ denote the training dataset, and let $\mathcal{L}(\cdot)$ denote the loss function that is nonnegative on $\Theta$, where $\Theta$ is the domain in which the parameters $\theta$ are optimized, and $\theta$ can be either a vector or a matrix. The objective is to find 
\begin{equation}\label{eq:theta_optim}
    \theta^* = \arg \min_{\theta \in \Theta} \mathcal{L}(\theta).
\end{equation}

As a well-known optimization method applicable to solve the problem~\eqref{eq:theta_optim} with $\Theta$ being Euclidean space without further restrictions, the discrete-time update of gradient descent \textit{with momentum}
\begin{align}\label{eq:dis_gd_mmt}
\theta  &\leftarrow \theta + \eta \phi \nonumber \\
\phi    &\leftarrow (1 - \gamma) \phi - \nabla \mathcal{L}(\theta)
\end{align}
corresponds to the following CTDS through the semi-implicit Euler method with step-size $\eta$:
\begin{equation}\label{eq:cont_gd_mmt}
\begin{bmatrix}
\dot \theta \\ \dot \phi
\end{bmatrix}
=
\begin{bmatrix}
\phi\\
(-\gamma \phi - \nabla \mathcal{L}(\theta))/\eta
\end{bmatrix},
\end{equation}
where $\phi$ represents the changing rate of $\theta$ and $\gamma > 0$ is the momentum coefficient.

However, the above systems \eqref{eq:dis_gd_mmt} and \eqref{eq:cont_gd_mmt} are only for Euclidean space and thus cannot be directly applied to the case when $\Theta$ is the Stiefel manifold. To design a discrete-time counterpart of~\eqref{eq:dis_gd_mmt} with orthogonality, we start from the construction of a CTDS defined on the tangent bundle of the Stiefel manifold that is analogous to~\eqref{eq:cont_gd_mmt} and represents the gradient descent process with a momentum-like term.

Specifically, we have the following CTDS:
\begin{equation}\label{eq:ds_original}
\begin{bmatrix}
\dot \theta \\ \dot \phi
\end{bmatrix}
=
\begin{bmatrix}
\phi\\
-\theta \phi^T \phi + (D - \theta \Sym(\theta^T D))/\eta
\end{bmatrix}
\end{equation}
with $\theta, \phi \in \R^{n \times p}$, where $\eta > 0$ will be later used in the discrete-time algorithm as the step-size and 
\begin{equation}\label{eq:def_D}
D = D(\theta, \phi) \coloneqq -\gamma \phi - \nabla \mathcal{L}(\theta).
\end{equation}
The following lemma shows that the dynamical system \eqref{eq:ds_original} is indeed well defined on $T \mathrm{St}(n,p)$. 
\begin{lemma}\label{lem:cont_ds_on_st}
For the dynamical system \eqref{eq:ds_original} defined on $\mathbb R^{n\times p} \times \mathbb R^{n\times p}$, if $\theta(0)^T \theta(0) = I$ and $\Sym(\theta(0)^T \phi(0)) = 0$, then $\theta(t)^T \theta(t) = I$ and $\Sym(\theta(t)^T \phi(t)) = 0$ hold for all $t \geq 0$.
\end{lemma}
\begin{proof}
Refer to the supplementary material for the proofs that are not given in the main text.
\end{proof}
    
The following lemma shows the asymptotic stability on $T \mathrm{St}(n,p)$ of the dynamical system~\eqref{eq:ds_original}.
\begin{lemma}\label{lem:converge_to_min}
    Assume that $\theta^* = \arg \min_{\theta \in \mathrm{St}(n,p)} \mathcal{L}(\theta)$ uniquely exists, and let $c_0 \geq 0$ such that $(\theta^*, 0)$ is the only point in $\set{(\theta, 0) \in \Omega: \operatorname{grad}_\theta \mathcal{L} (\theta) = 0}$, where $\Omega = \set{(\theta, \phi) \in T\mathrm{St}(n,p): \frac{\eta}{2} \norm{\phi}^2 + \mathcal{L}(\theta) \leq c_0}$. {Assume that $\mathcal{L}$ is $C^2$ on $\set{\theta \in \mathrm{St}(n,p): \mathcal{L}(\theta) \leq c_0}$.} Then each trajectory of \eqref{eq:ds_original} starting in $\Omega$ stays in $\Omega$ for all forward time and asymptotically converges to $(\theta^*, 0)$ as time tends to infinity.
\end{lemma}
Refer to the supplementary material for a local version of Lemma~\ref*{lem:converge_to_min}.

It now remains to discretize the system~\eqref{eq:ds_original}. However, we cannot just use usual discretization techniques on Euclidean space since the trajectories of the discretized system may go off the manifold, and the usual structure-preserving discretization methods can be complicated and computationally expensive.
To address these problems, we propose to extend system~\eqref{eq:ds_original} to an ambient Euclidean space and modify it by adding a feedback term so that the original domain becomes locally attractive.
After doing so, we can apply any off-the-shelf Euclidean discretization method to get our final optimization algorithm with the stability carried over to the discretized system.

Let $W = S \times \R^{n \times p}$, where
\begin{equation}
S = \set{\theta \in \R^{n \times p}: \norm{\theta^T \theta - I} < 1}
\end{equation}
is an open neighborhood of $\mathrm{St}(n,p)$ such that $\theta^T \theta$ is invertible for all $\theta \in S$. 
Consider the following dynamical system defined on $W$:
\begin{align}\label{eq:ds_ext_fb}
\begin{bmatrix}
\dot \theta \\ \dot \phi
\end{bmatrix}
=
X (\theta, \phi)
- \frac{\alpha}{4}
\begin{bmatrix}
\theta(I - (\theta^T \theta)^{-1})\\
\theta(\theta^T \theta)^{-1} (\phi^T \theta (\theta^T \theta)^{-1} + \theta^T \phi)
\end{bmatrix},
\end{align}
where 
\begin{align*}
X (\theta, \phi) & =
\begin{bmatrix}
X_\theta (\theta, \phi) \\
X_\phi (\theta, \phi)
\end{bmatrix} \nonumber \\
& =
\begin{bmatrix}
\phi - \theta(\theta^T \theta)^{-1} \Sym(\theta^T \phi)\\
\theta(\theta^T \theta)^{-1} ((\theta^T \theta)^{-1} \theta^T \phi \Sym(\theta^T \phi) -\phi^T\phi) +\\ (D - \theta(\theta^T \theta)^{-1} \Sym(\theta^T D))/\eta
\end{bmatrix}
\end{align*}
with $D$ defined in \eqref{eq:def_D}, and $\alpha > 0$ is the feedback coefficient. Note that \eqref{eq:ds_ext_fb} is identical to \eqref{eq:ds_original} on $T \mathrm{St}(n,p)$.  
Let $V: \R^{n \times p} \times \R^{n \times p} \rightarrow \R_{\geq 0}$ be a function defined as
\begin{equation}\label{eq:V_func}
V(\theta, \phi) = \frac{k_1}{4}\norm{\theta^T \theta - I}^2 + \frac{k_2}{2}\norm{\Sym(\theta^T \phi)}^2,
\end{equation}
where $k_1, k_2 > 0$. We have $V^{-1}(0) = T \mathrm{St}(n,p)$
and
\begin{align}
\nabla V (\theta, \phi) &=
\begin{bmatrix}
\nabla_\theta V (\theta, \phi) \\
\nabla_\phi V (\theta, \phi)
\end{bmatrix} \nonumber \\ 
&=
\begin{bmatrix}
k_1\theta(\theta^T \theta - I) + k_2\phi \Sym(\theta^T \phi) \\
k_2\theta \Sym(\theta^T \phi)
\end{bmatrix}.  
\end{align}
We have the following theorem showing that the tangent bundle of the Stiefel manifold is a locally attractive submanifold of the dynamical system~\eqref{eq:ds_ext_fb} on $W$ with exponential stability.
\begin{theorem}\label{thm:main_thm}
For each $0 < c < k_1/4$, $V^{-1}([0,c])$ is a subset of $W$ and each trajectory of \eqref{eq:ds_ext_fb} starting in $V^{-1}([0,c]) \subset W$ stays in $V^{-1}([0,c])$ for all forward time and exponentially converges to $V^{-1}(0) = T\mathrm{St}(n,p)$ as time tends to infinity.
\end{theorem}
\begin{proof}
To proof the theorem, we make use of the following lemmas.
\begin{lemma}\label{lem:crit_pts}
For each $0 < c < k_1/4$, the set of all critical points of $V$ in $V^{-1}([0, c])$ is $V^{-1}(0)$.
\end{lemma}

\begin{lemma}\label{lem:inp_zero}
$\inp{\nabla V(\theta, \phi), X(\theta, \phi)} = 0, \forall (\theta, \phi) \in W$.
\end{lemma}

We define a function $L: W \rightarrow \R^{2n \times 2n}$ as
\begin{align}
L(\theta,\phi) 
= 
\begin{bmatrix}
\frac{1}{4k_1}\theta (\theta^T\theta)^{-2}\theta^T & L_1(\theta, \phi) \\
L_1^T(\theta, \phi) & L_2(\theta, \phi)
\end{bmatrix}
\end{align}
where $L_1(\theta, \phi) = -\frac{1}{4k_1}\theta(\theta^T \theta)^{-2} \theta^T \phi (\theta^T \theta)^{-1} \theta^T$ and $L_2(\theta, \phi) = \frac{1}{4k_1} \theta(\theta^T \theta)^{-1} \phi^T \theta (\theta^T \theta)^{-2} \theta^T \phi(\theta^T \theta)^{-1} \theta^T + \frac{1}{2k_2}\theta (\theta^T\theta)^{-2}\theta^T$.

\begin{lemma}\label{lem:nablaV_V}
    $\inp{\nabla V(\theta,\phi), L(\theta,\phi)\nabla V(\theta, \phi)} = V(\theta,\phi), \forall (\theta, \phi) \in W$. 
\end{lemma}

By direct computation, it is easy to check that \eqref{eq:ds_ext_fb} can be written as
\begin{align}
\begin{bmatrix}
\dot \theta \\ \dot \phi
\end{bmatrix}
&=
X (\theta, \phi) - \alpha L(\theta,\phi) \nabla V(\theta,\phi).
\end{align}
By Lemma \ref*{lem:inp_zero} and \ref*{lem:nablaV_V}, along each trajectory $(\theta (t), \phi(t))$ of \eqref{eq:ds_ext_fb} starting in $V^{-1}([0,c]) \subset W$, 
\begin{equation}
\frac{d}{dt} V = \inp{\nabla V, X -\alpha L \nabla V} = \inp{\nabla V, -\alpha L \nabla V} = -\alpha V,
\end{equation}

which implies that 
\begin{equation}
V(\theta(t), \phi(t)) = e^{-\alpha t} V(\theta(0), \phi(0))
\end{equation}

for all $t\geq 0$. Combined with Lemma~\ref*{lem:crit_pts}, this completes the proof.
\end{proof}

\begin{remark}
    Informally speaking, Lemma~\ref*{lem:inp_zero} says that the vector field $X$ does not change the value of $V$ and Lemma~\ref*{lem:nablaV_V} implies that the added term $-\alpha L \nabla V$ keeps decreasing the value of $V$. Therefore, the dynamical system~\eqref{eq:ds_ext_fb} consisting of these two parts can converge to the tangent bundle of the Stiefel manifold where $V = 0$ and dynamical system~\eqref{eq:ds_ext_fb} coincides with dynamical system~\eqref{eq:ds_original}. Moreover, this allows us to perform usual Euclidean discretization on system~\eqref{eq:ds_ext_fb} with the trajectories of the corresponding discretized dynamical system kept close to the tangent bundle of the Stiefel manifold.
\end{remark}

By straightforward discretization of the system~\eqref{eq:ds_ext_fb} with the semi-implicit Euler method with step size $\eta$, we have the discretized Algorithm~\ref*{alg:FGD} as the update rule for the parameters with orthogonality.

In the practical application of Algorithm~\ref*{alg:FGD}, we approximately use $2I - \theta^T\theta$ for $(\theta^T\theta)^{-1}$, which is a truncation of the Neumann series \cite{zhu2015matrix}.
This approximation does not cause great loss of precision because $\theta^T\theta$ indeed stays around the identity matrix and it can considerably reduce the computational complexity as the inversion operation is time-consuming, especially for large matrices.
\begin{theorem}
    With the approximation $(\theta^T\theta)^{-1} \approx 2I - \theta^T\theta$, the additional time complexity of Algorithm~\ref*{alg:FGD} is $O(p^2 n)$.
\end{theorem}
FGD uses only matrix multiplications while the previous methods additionally use computationally expensive SVD, QR decomposition, etc.
In general, this approximation is inappropriate. For example, the Cayley transform also requires the computation of matrix inversion, but the inversion is not necessarily around the identity matrix.
It is notable that even if one directly uses inversion operation, it will inevitably have numerical errors when using any existing deep learning framework.
\begin{theorem}
    Assume there exists a constant $c > 0$ such that $\norm{\phi} \leq c$ for all timesteps when using Algorithm~\ref*{alg:FGD} with the approximation $(\theta^T\theta)^{-1} \approx 2I - \theta^T\theta$. For any given $0 < \epsilon < 1$, there exist $\eta = \eta(\epsilon) > 0$ and $\alpha = \alpha(\epsilon) > 0$ so that $\norm{\theta^T \theta - I} \leq \epsilon$ holds for all timesteps.
\end{theorem}
The rationality of this approximation is also bolstered by the experimental results in Section~\ref*{sec:exp} in terms of accuracy, complexity, and numerical stability.

\begin{algorithm}[t!]
\caption{Feedback Gradient Descent}
\label{alg:FGD}
\begin{algorithmic}
\STATE {\bfseries Input:} learning rate $\eta$, loss function $\mathcal L$, momentum coefficient $\gamma$, feedback coefficient $\alpha$
\STATE Initialize $\theta$ and $\phi$ such that $\theta^T \theta = I$, and $\phi = \theta \Sym(\theta^T \nabla \mathcal L(\theta)) - \nabla \mathcal L(\theta)$
\WHILE {training}
\STATE {$\theta \leftarrow \theta + \eta [X_\theta(\theta, \phi) - \frac{\alpha}{4} \theta(I - (\theta^T \theta)^{-1})]$}
\STATE {$\phi \leftarrow \phi + \eta [X_\phi(\theta, \phi) - \frac{\alpha}{4} \theta(\theta^T \theta)^{-1} (\phi^T \theta (\theta^T \theta)^{-1} + \theta^T \phi)]$}
\ENDWHILE
\end{algorithmic}
\end{algorithm}

\begin{figure*}[tb]
    \centering
    \begin{subfigure}{0.45\textwidth}
        \centering
        \includegraphics[width=0.9\columnwidth]{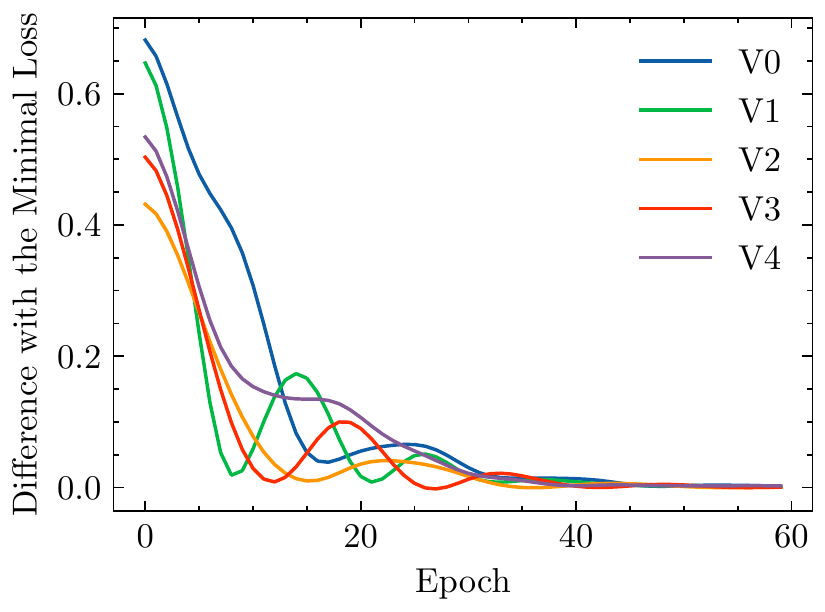}
        \caption{Difference with the minimal possible distance.}
        \label{toy_loss}
    \end{subfigure}%
    \begin{subfigure}{0.45\textwidth}
        \centering
        \includegraphics[width=0.9\columnwidth]{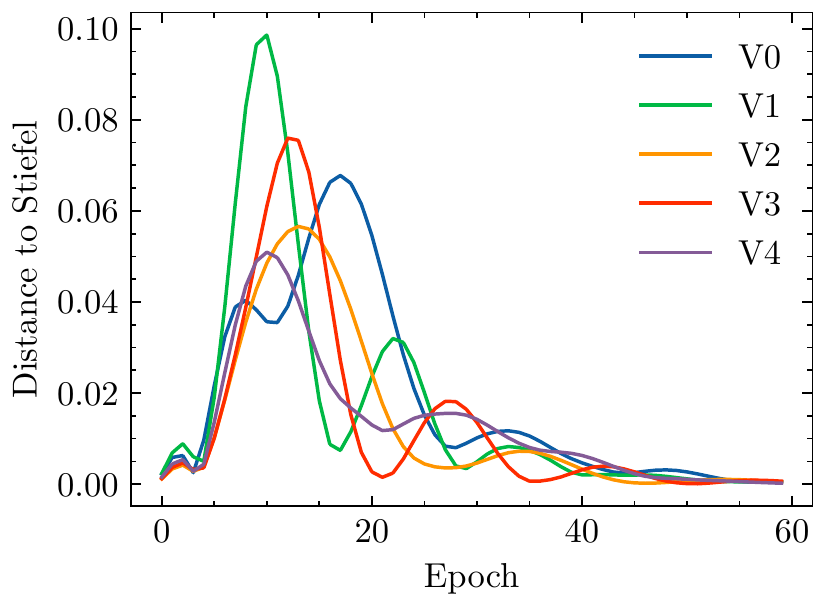}
        \caption{Distance to the Stiefel manifold.}
        \label{toy_dist}
    \end{subfigure}
    \caption{The results of the toy example.}
\end{figure*}

\section*{Experiments}\label{sec:exp}
We conduct comprehensive experiments using various models such as WideResNet, ResNet, VGG, and PreActResNet, on CIFAR-10/100, SVHN, and ImageNet datasets. We follow the official settings for the training-testing split of the datasets, the pre-processing, the data augmentation, etc.

To evaluate the performance, we compare our method with stochastic gradient descent (SGD) \textit{with momentum} and several state-of-the-art optimization algorithms imposing orthogonality such as OCNN \cite{Wang_2020_CVPR}, ONI \cite{huang2020controllable}, OMDSM \cite{huang2018orthogonal}, Cayley SGD/ADAM \cite{li2020efficient}, SRIP \cite{bansal2018can}, SVB \cite{jia2017improving}, and LCD \cite{rodriguez2016regularizing}.

When we list the experimental results, those without $*$ are the results claimed in the original papers and those with $*$ are obtained by running the corresponding open-source code on our machines.

All the experiments are on the image classification task and all the models used are convolutional neural networks (CNNs). In CNNs, the weight parameter of a convolutional layer has a size of $c_o \times c_i \times K_1 \times K_2$, where $c_o$ and $c_i$ are the channel numbers of output and input, respectively, and $(K_1, K_2)$ is the size of the kernel.
To impose orthogonality with FGD, we first need to reshape each such parameter into a two-dimensional matrix. This can be simply done by flattening the last three dimensions and transposing it, which gives a matrix with a size of $p_i \times c_o$, where $p_i = c_i K_1 K_2$. We only impose orthogonality on parameters with $p_i \geq c_o$ and $\min(K_1, K_2) > 1$, which usually constitute most part of a CNN model.
For the specific settings for each sub-experiment, refer to the supplementary material for the details not mentioned in the main text.
Our implementation uses PyTorch \cite{NEURIPS2019_9015}.

\subsection{Toy Example}
We first try FGD on a toy example. Fix a non-singular $V \in \R^{n \times p}$ and consider the following optimization problem \cite{huang2018orthogonal} that finds the matrix on the Stiefel manifold with minimal distance to $V$:
$\min_{W \in \mathrm{St}(n,p)} f(W) = \norm{W - V}.$
The closed-form solution of the problem is
$W_0 (V) = V D \Lambda^{-1/2} D^T,$
where $\Lambda = \mathrm{diag}(\{\lambda_1, ..., \lambda_p\})$ is a diagonal matrix corresponding to the eigenvalues of $S = V^T V$ and $D$ is the matrix consisting of the corresponding eigenvectors. Let $f_0 = \norm{W_0 - V}$ be the minimum distance. We apply FGD on this problem with learning rate $\eta = 0.1$, momentum coefficient $\gamma = 0.1$, and feedback coefficient $\alpha = 12$. We pick $5$ different random $V \in \R^{5 \times 3}$. In Figures\footnote{We use SciencePlots \cite{SciencePlots} to draw the plots.}~\ref*{toy_loss} and \ref*{toy_dist}, for the variable $W$ at the current epoch, we show the difference $f(W) - f_0$ between the current $f$ value and the minimum, and the distance $\norm{W^T W - I}$ to the Stiefel manifold during the training of $60$ epochs. In each trial, FGD decreases the loss of the optimization problem and maintains the orthogonality at the same time, practically showing the correctness of FGD.

\subsection{WideResNet on CIFAR-10/100}\label{exp_cif_wres}
We use the WideResNet 28-10 model on CIFAR-10/100 datasets.
We provide the most detailed results for this sub-experiment since the results of most baseline methods are available, which allows us to do the most comprehensive comparisons.
All experiments using WideResNet 28-10 are done on one TITAN Xp GPU.

In Table~\ref*{res_wres_cif}, we list the test accuracy rates, the training time, and the distances to the Stiefel manifold of the transformed two-dimensional parameters in the final trained models of each method, where FGD shows outstanding performance in terms of accuracy, efficiency, and numerical stability.

\begin{figure*}[tb]
\centering
\begin{subfigure}{0.45\textwidth}
    \centering
    \includegraphics[width=0.9\columnwidth]{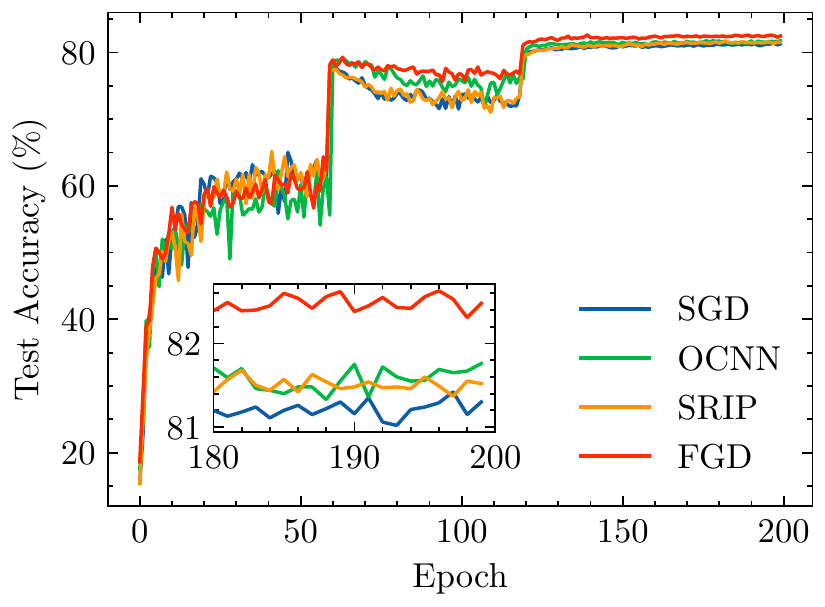}
    \caption{Test accuracy rates.}
    \label{fig_wres_cif100_test_acc}
\end{subfigure}%
\begin{subfigure}{0.45\textwidth}
    \centering
    \includegraphics[width=0.9\columnwidth]{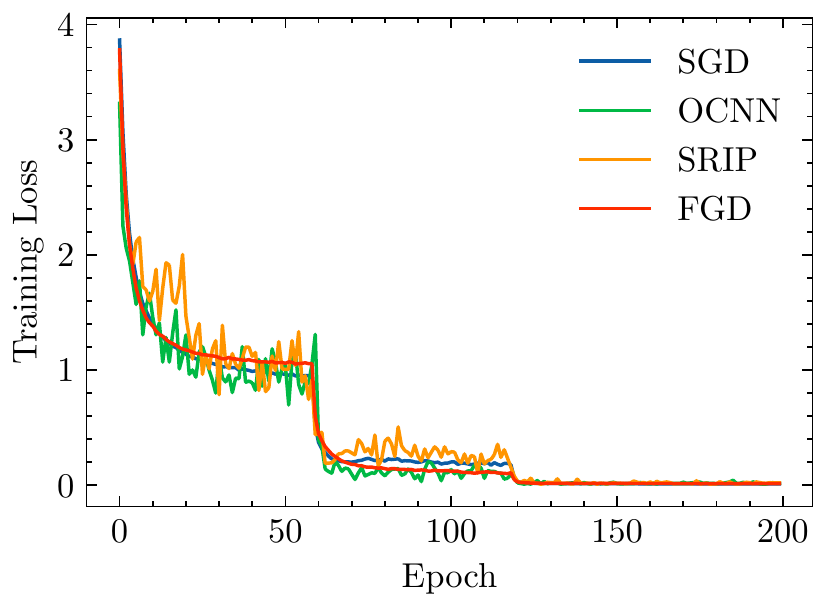}
    \caption{Training losses.}
    \label{fig_wres_cif100_train_loss}
\end{subfigure}
\caption{The results of the experiments using WideResNet 28-10 on CIFAR-100.}
\end{figure*}

\begin{table*}[tb]
\centering
\renewcommand{\tabcolsep}{10pt}
\begin{tabular}{lcccccc}
Method & \small{CIFAR-10} & \small{CIFAR-100} & \small{Training time} & $L_1$ & $L_2$ & $L_3$ \\
\midrule
SGD* & 96.27 $\pm$ 0.054 & 81.47 $\pm$ 0.282 & 110.16 & 12.26 & 15.80 & 23.98 \\
OCNN* & 96.34 $\pm$ 0.125 & 81.50 $\pm$ 0.156 & 144.28 & 3.00e-4 & 3.70 & 22.14 \\
ONI* & 96.03 $\pm$ 0.153 & 81.18 $\pm$ 0.107 & 132.94 & 12.26 & 17.08 & 22.42 \\    
\small{Cayley SGD} & 96.34 & 81.74 & 218.70\small{(292.73*)} & 4.70e-6* & 4.12e-6* & 7.32e-6* \\
\small{Cayley ADAM} & 96.43 & 81.90 & 224.40\small{(292.89*)} & 3.50e-6* & 4.48e-6* & 4.52e-6* \\
OMDSM & 96.27 & 81.39 & 312.45 & N/A & N/A & N/A \\
SRIP & 96.40 & 81.81 & 133.72* & 12.57* & 15.62* & 24.00* \\
SVB & 96.42 & 81.68 & N/A & N/A & N/A & N/A \\
LCD & 96.31 & 81.44 & N/A & N/A & N/A & N/A \\
FGD*(ours) & \textbf{96.51} $\pm$ 0.062 & \textbf{82.25} $\pm$ 0.220 & 136.71 & 2.12e-6 & 3.87e-6 & 6.56e-6 \\
\end{tabular}
\caption{Test accuracy rates (in percentages), training times per epoch (in seconds), and the distances to the Stiefel manifold of the transformed parameters using WideResNet 28-10 on CIFAR-10/100.}
\label{res_wres_cif}
\end{table*}

\noindent \textbf{Accuracy.}
We report the mean value and standard deviation of the test accuracy rates on both CIFAR-10 and CIFAR-100 datasets when they are available.
We list the claimed statistics in the original paper when we fail to reproduce similar ones.
As some statistics are not reported in the original papers and the corresponding code is unavailable, we leave them as N/A.
FGD achieves the highest accuracy rates on both CIFAR-10 and CIFAR-100 with accuracy gains of $0.24\%$ and $0.78\%$ over SGD, respectively.
Specifically, FGD achieves an average accuracy rate of $82.25\%$ and is the only method with an accuracy rate of more than $82\%$ on CIFAR-100.
The second best method is Cayley SGD showing an average accuracy rate of $81.90\%$. However, the gap with our method is favorably obvious.
Figures \ref*{fig_wres_cif100_test_acc} and \ref*{fig_wres_cif100_train_loss} show the test accuracy rates and training losses during the training process on CIFAR-100.
Specifically, in Figure~\ref*{fig_wres_cif100_test_acc}, we zoom in the accuracy rates in the last $20$ epochs, where the superiority of FGD over others in terms of accuracy is clearly shown. Besides, the ability of FGD to mitigate the drop in the test accuracy during epochs $60$ and $120$ is observed.
As shown in Figure~\ref*{fig_wres_cif100_train_loss}, the curve of FGD is much smoother than those of OCNN and SRIP, showing that FGD not only improves the accuracy rate but also stabilizes the training.
Note that SRIP and OCNN include regularization terms imposing orthogonality into the loss function, but in Figure~\ref*{fig_wres_cif100_train_loss} we only compare the loss of the optimization problem for the sake of clarity.

\noindent \textbf{Efficiency.}
We compare the efficiency by considering the training time for each method, as done in a previous work \cite{li2020efficient}.
The efficiency of FGD is comparable to those of ONI, SRIP, and OCNN.
All three methods are based on soft constraints and fail to keep orthogonality consistently.
The training time of FGD is only $2.2\%$-$2.8\%$ longer than that of ONI or SRIP.
Moreover, FGD is even $5.2\%$ faster than OCNN.
Compared with Cayley SGD/ADAM and OMDSM based on hard constraints, FGD is $37.5\%$-$56.2\%$ faster and achieves better numerical stability at the same time.

\noindent \textbf{Numerical stability.}
The distance to the Stiefel manifold of each method is used to compare the numerical stability.
We choose three representative layers, where the first one $L_1$ has original size $(160, 160, 3, 3)$ and transformed size $1440 \times 160$; the second one $L_2$ has original size $(320, 320, 3, 3)$ and transformed size $2880 \times 320$; and the third one $L_3$ has original size $(640, 640, 3, 3)$ and transformed size $5760 \times 640$. For a transformed matrix $M$ with size $n \times p$, the distance to the Stiefel manifold is computed as $\norm{{M}^T {M} - I_p}$.
FGD achieves the lowest distance to the Stiefel manifold on each representative layer, showing the highest numerical stability when imposing orthogonality.

\begin{table*}[tb]
\centering
\begin{tabular}{lccccc}
Method      & ResNet18          & ResNet34      & ResNet50      & ResNet101         & ResNet110 \\
\midrule
SGD*        & 95.20/77.85       & 95.30/78.29   & 95.46/80.29   & 95.37/79.93       & 94.36/74.79 \\
SRIP 
            & N/A               & N/A           & N/A           & N/A               & 93.45\small{(94.44*)}/74.99 \\
OCNN* 
            & 95.20*/78.10      & 95.41*/78.70  & 95.54*/80.33* & 95.45*/80.12*     & 94.89*/75.93* \\
FGD*(ours) &\textbf{95.57/78.87}& \textbf{95.52/79.52} & \textbf{95.83/80.64} & \textbf{95.68/80.17} & \textbf{95.20/76.33}   \\
\end{tabular}
\caption{Test accuracy rates using ResNet on CIFAR-10/100.}
\label{res_resnet_cif}
\end{table*}

\subsection{ResNet and VGG on CIFAR-10/100}\label{exp_cif_res}
We use ResNet and VGG models on CIFAR-10/100.
We use ResNet110 with bottleneck layers.
We use VGG models with batch normalization \cite{ioffe2015batch}.
For each of the models, orthogonality is imposed only on the convolutional layers in the last two residual modules.
The parameters of these layers constitute the majority of the whole network.
This restriction on the range of parameters to impose orthogonality shows the best performance.
We also apply the same restriction when using OCNN for fairness. This restriction also improves the performance of OCNN.

In Tables~\ref*{res_resnet_cif} and \ref*{res_vgg_cif}, we report the test accuracy rates when using ResNet and VGG, respectively.
We focus on the comparisons in terms of accuracy as the statistics on the training time and the distances to the Stiefel manifold of most methods are unavailable.
In each cell of both tables, the number on the left represents the accuracy on CIFAR-10 and that on the right is for CIFAR-100.
In Table~\ref*{res_resnet_cif}, we compare FGD with three methods: SGD, SRIP, and OCNN.
The statistics and code of SRIP are only available for ResNet110.
FGD consistently outperforms others in terms of accuracy. Notably, the accuracy gains of FGD over SGD are more than $1\%$ when using ResNet18, ResNet34, and ResNet110 on CIFAR-100.
In Table~\ref*{res_vgg_cif}, we compare FGD with three methods: SGD, Cayley SGD, and Cayley ADAM.
FGD consistently outperforms others in terms of accuracy with accuracy gains of $0.97\%$, $0.52\%$, and $1.38\%$ over SGD when using three different VGG models on CIFAR-100.

\begin{table}[tb]
\centering
\renewcommand{\tabcolsep}{3pt}
\begin{tabular}{lcccc}
Method      & VGG13     & VGG16     & VGG19 \\
\midrule
SGD*        & 93.92/74.27 & 93.81/74.03 & 93.73/72.99 \\
\small{Cayley SGD}
            & 94.10/75.14 & 94.23/74.52 & 94.15/74.32 \\
\small{Cayley ADAM}
            & 94.07/74.90 & 94.12/74.39 & 93.97/74.30 \\
FGD*(ours)  & \textbf{94.32/75.24} & \textbf{94.45/74.55} & \textbf{94.22/74.37} \\
\end{tabular}
\caption{Test accuracy rates using VGG on CIFAR-10/100.}
\label{res_vgg_cif}
\end{table}

\subsection{ResNet and PreActResNet on ImageNet}
To further show the performance of FGD, we conduct experiments using ResNet and PreActResNet models on the more complicated ImageNet ILSVRC-2012 dataset, where we compare FGD with four methods: SGD, OCNN, ONI, and SRIP.
In Table~\ref*{res_imagenet}, we report the test accuracy rates of different methods.
In each cell, the number on the left represents the top-1 accuracy and that on the right is the top-5 accuracy. As some statistics are not reported in the original paper and the corresponding code is unavailable, we leave them as N/A.
Among the five methods, FGD achieves the best performance in terms of both top-1 and top-5 accuracy rates using ResNet50 and PreActResNet34.
In the experiments using ResNet34, FGD achieves the highest top-1 accuracy, and the achieved top-5 accuracy is only lower than the claimed one of OCNN and higher than all the others including the reproduced one of OCNN.
The top-1 accuracy gains of FGD over SGD are $0.47\%$ on ResNet34, $0.67\%$ on ResNet50, and $0.91\%$ on PreActResNet34, while the top-5 accuracy gains are $0.46\%$, $0.35\%$, and $0.46\%$.

\begin{table}[tb]
\centering
\renewcommand{\tabcolsep}{3pt}
\begin{tabular}{lccc}
Method      & ResNet34      & ResNet50          & \small{PreActRes34}\\
\midrule
SGD*        & 73.49, 91.31  & 76.13, 92.98      & 72.62, 90.95\\
\multirow{2}{*}{OCNN}        & 73.93, \textbf{92.11} & \multirow{2}{*}{76.38, 93.18*}      & \multirow{2}{*}{N/A} \\
 & (73.81*, 91.63*)  \\
ONI*        & 73.68, 91.60   & 76.75, 93.28       & N/A\\
SRIP        & N/A, 91.68    & N/A, 93.13          & N/A, 91.21\\
FGD*(ours)  & \textbf{73.96}, 91.77     & \textbf{76.80, 93.33}    & \textbf{73.53, 91.41}\\
\end{tabular}
\caption{Test accuracy rates on ImageNet.}
\label{res_imagenet}
\end{table}





\section*{Conclusion and Discussion}\label{sec:conc_and_disc}
In this paper, we have proposed Feedback Gradient Descent (FGD), an efficient and stable optimization algorithm with orthogonality for DNNs.
Inspired by Feedback Integrators, we have constructed a continuous-time dynamical system in a Euclidean space containing the tangent bundle of the Stiefel manifold as a local attractor, and completed the discretization with the semi-implicit Euler method.
The excellent performance of FGD in terms of accuracy, efficiency, and numerical stability has been shown through the theoretical analyses and the extensive experiments.
It is hard to apply most existing momentum-based methods on manifolds to DNNs due to their problems on efficiency \cite{lezcano2020adaptive}.
Moreover, some recent works show that the existing momentum-based methods suffer from insufficiency \cite{kidambi2018insufficiency, liu2019accelerating}.
FGD provides a practical and efficient solution for generalizing momentum to manifolds.

FGD has been shown to be advantageous on the image classification task. Additionally, for many other tasks where orthogonality has manifested its ability, such as graph embedding \cite{robles2007riemannian, shaw2009structure, liu2017semi} and matrix factorization \cite{ding2006orthogonal, zhang2016efficient}, FGD can hopefully also be utilized.
Although we have designed FGD specifically for the Stiefel manifold, our framework can be flexibly applied to other submanifolds embedded in Euclidean space, such as the Oblique manifold \cite{huang2017projection}.
Some recent works \cite{xiao2018dynamical, qi2020deep} propose techniques to leverage isometry and orthogonality for training DNNs, without using normalization or skip connections.
Another recent work \cite{liu2021orthogonal} proposes an orthogonal over-parameterized training framework by learning an orthogonal transformation to help effectively train DNNs.
We leave it as future works to explore the potential of FGD in these directions.



\section*{Acknowledgements}
The authors appreciate the constructive and insightful comments from the reviewers.
The authors also would like to thank Lin Wang and Yi-Ling Qiao for discussions on the PyTorch implementation. 

This work was conducted by the Center for Applied Research in Artificial Intelligence (CARAI) grant funded by Defense Acquisition Program Administration (DAPA) and Agency for Defense Development (ADD) [UD190031RD].

\bibliography{ref}

\end{document}